\documentclass[sn-mathphys,Numbered]{sn-jnl}

\usepackage{graphicx}%
\usepackage{multirow}%
\usepackage{amsmath,amssymb,amsfonts}%
\usepackage{amsthm}%
\usepackage{mathrsfs}%
\usepackage[title]{appendix}%
\usepackage{xcolor}%
\usepackage{textcomp}%
\usepackage{manyfoot}%
\usepackage{booktabs}%
\usepackage{algorithm}%
\usepackage{algorithmicx}%
\usepackage{algpseudocode}%
\usepackage{listings}%

\theoremstyle{thmstyleone}%
\newtheorem{theorem}{Theorem}
\theoremstyle{thmstyletwo}%

\theoremstyle{thmstylethree}%

\usepackage{bbm} 
\usepackage[shortlabels]{enumitem}

\begin{document}

\title[Linearly-scalable learning of smooth low-dimensional patterns with permutation-aided entropic dimension reduction]{Linearly-scalable learning of smooth low-dimensional patterns with permutation-aided entropic dimension reduction}


\author*[1]{\fnm{Illia} \sur{Horenko}}\email{horenko@rptu.de}

\author[1,2]{\fnm{Lukas} \sur{Pospisil}}\email{lukas.pospisil@vsb.cz}

\affil*[1]{\orgdiv{Chair for Mathematics of AI, Faculty of Mathematics}, \orgname{RPTU Kaiserslautern-Landau}, \orgaddress{\street{Gottlieb-Daimler-Str. 48}, \city{Kaiserslautern}, \postcode{67663}, \country{Germany}}}

\affil[2]{\orgdiv{Department of Mathematics, Faculty of Civil Engineering}, \orgname{VSB – Technical University Ostrava}, \orgaddress{\street{Ludvika Podeste 1875/17}, \city{Ostrava-Poruba}, \postcode{70800}, \country{Czech Republic}}}


\abstract{In many data science applications, the objective is to extract appropriately-ordered smooth low-dimensional data patterns from high-dimensional data sets. This is challenging since common sorting algorithms are primarily aiming at finding monotonic orderings in low-dimensional data, whereas typical dimension reduction and feature extraction algorithms are not primarily designed for extracting smooth low-dimensional data patterns. We show that when selecting the Euclidean smoothness as a pattern quality criterium, both of these problems  (finding the optimal 'crisp' data permutation and extracting the sparse set of permuted low-dimensional smooth patterns) can be efficiently solved numerically as one unsupervised entropy-regularized iterative optimization problem. We formulate and prove the conditions for monotonicity and convergence of this linearly-scalable (in dimension) numerical procedure, with the iteration cost scaling of $\mathcal{O}(DT^2)$, where $T$ is the size of the data statistics and $D$ is a feature space dimension.  The efficacy of the proposed method is demonstrated through the examination of synthetic examples as well as a real-world application involving the identification of smooth bankruptcy risk minimizing transition patterns from high-dimensional economical data. The results showcase that the statistical properties of the overall time complexity of the method exhibit linear scaling in the dimensionality $D$ within the specified confidence intervals.}

\maketitle

\section{Introduction}
\label{sec:introduction}

The problem of efficiently re-arranging or permuting data in some desired way, for example, deploying sorting algorithms, belongs to the most long-standing questions in computer science and applied mathematics \cite{cormen01}. Many  impressive theoretical and practical algorithmic results in this area could be established in past decades for problems of one- and low-dimensional data sorting,  where one seeks for a monotonic (ascending or descending) ordering in one  or several  data  dimensions \cite{master80,cormen01,sort02}.  In their seminal work, Garrett Birkhoff and John von Neumann have established a mathematical relationship between permutations of the $T$-component vector  and the  multiplication of this  vector with $T\times T$ double-stochastic Markovian operator $P$, containing only one 1.0 in every row and every column, with all other matrix elements being equal to zero \cite{birhoff46,neumann53}. Such a 'crisp'  double-stochastic Markov operator (containing only zeroes and ones) is referred to as a permutation matrix - as opposed to the 'fuzzy' double-stochastic Markov operators that contain elements that are between zero and one.   For a vector with $T$ elements there exist $T!$ of all possible ('crisp') permutation matrices $P$,  that build the edges of the $T(T-2)$-dimensional polytope of all 'fuzzy' (or 'soft') double-stochastic Markov matrices \cite{ziegler_95,markov_book17}. NP-complexity of the original 'crisp' permutation problem - arising when extending the \emph{generic} sorting and permutation problems (satisfying desired criteria) from one to several dimensions - has led to a growing popularity of methods based on 'soft' relaxations of the permutation matrix, ignited by several very successful mathematical and algorithmic approaches that dwell on the spectral decomposition of the 'soft'/'fuzzy' Markov operator \cite{Deuflhard2000,schuette10} and allowing for a metastable decomposition and a reduced analysis of high-dimensional systems from various areas \cite{SchuetteSarich2013,roeblitz13, horenko15}. Further, the ideas of 'soft' Markovian relaxation for permutations were explored in the areas of graph-matching and graph-alignment, leading to new approaches to these problems - like the new spectral criteria for checking the matching of this 'fuzzy' /continuous relaxation to the original 'crisp' graph-matching permutation \cite{bronstein15}. These 'soft' permutation ideas were further applied to the supervised graph-permutation and graph-alignment problems  \cite{mena17,mena20,nikol23}.  In the literature, it is argued that the 'soft' permutations allow reducing NP-hard to P-hard algorithmic solutions, but at the same time is not clear how the loss of 'crispness' for  the resulting 'soft' permutation matrices, can avoid leading to such  'soft' permutation relaxation extremes as the stochastic matrices with all of the elements being equal to $\frac{1}{T}$ - and where instead of searching for the data permutations one finds a data average. 

In the following, we show that the particular problem of finding relevant low-dimensional subsets - together with  finding 'crisp' permutations that lead to smooth patterns in these low-dimensional subsets - can be solved together with a linear (in dimension $D$) complexity scaling, considering the combined problem as the problem of entropy-regularized expected non-smoothness minimization, where the expectation is taken over the a priori unknown feature probabilities $W_1,\dots,W_D$. Learning of these probabilities  $W_1,\dots,W_D$ will be derived from a joint optimization problem formulation - and performed using the analytically-solvable formula, computed together with the learning of the 'crisp' data permutations $P$.   


\section{Problem formulation}
\label{sec:problem}

Let $X\in\mathbb{R}^{T,D}$ be a $D$-dimensional data matrix with $T$ data instances, i.e., with every column  $X_{:,d}$ of this data matrix representing a statistics of $T$ instances in a feature dimension $d$, where $d=1,\dots,D$.  To define the permuted pattern smoothness measure, we first introduce the differencing operator $\Delta\in\mathbb{R}^{T,T}$, with all the elements being zeros, except the main diagonal that is equal to $(-1)$ and the upper diagonal being $1$. Multiplying any $T$-component column vector with this operator $\Delta$ computes differences between the neighboring elements of the vector. Without a loss of generality, in the following applications, we will consider the case of periodic boundary conditions for the data $X_{:,d}$, implying that in any dimension $d$ the first data element $X_{1,d}$ is a direct neighbor of the last data element $X_{T,d}$, meaning that also $\Delta_{T,1}=1$. We will refer to the matrix $\Delta$ that satisfies these conditions as a \emph{$\Delta$ satisfying a periodic boundary condition}.    

It is straightforward to verify that the Euclidean pattern non-smoothness $S_d$ in any data dimension $d$ can then be measured as the following non-negative scalar-valued expression:
\begin{eqnarray}\label{eq:functional_orig}
	S_d&=&\|\Delta P X_{:,d}\|_2^2=  X_{:,d}^\dagger P^\dagger\Delta^\dagger\Delta P X_{:,d},
\end{eqnarray}
where $\dagger$ denotes a transposition operation. Smooth permutations $ P X_{:,d}$ are characterized by low values of $S_d$, and increasing non-smoothness results in a growing value of this measure  (\ref{eq:functional_orig}). 

The central message of this brief report is in showing that the multidimensional extension of this one-dimensional non-smoothness measure (\ref{eq:functional_orig}) can be computationally efficiently formulated as an expected value with respect to the (a priori unknown) $D$-dimensional feature probability distribution $W=\left(W_1,W_2,\dots,W_D\right)$.  And that this unknown distribution $W$, together with the permutation matrix $P$ - can be very efficiently learned by solving the following Shannon entropy-regularized minimization problem: 
\begin{eqnarray}\label{eq:functionalT1}
	 \left\{W^\epsilon,P^\epsilon\right\}&=&\arg\min\limits_{W\in\mathbb{P}^{D},P\in\mathbb{P}^{T\times T}} F^\epsilon_{W}\left[S^\epsilon\right] (W,P), \nonumber\\
	  F^\epsilon_{W}\left[S^\epsilon\right](W,P)&=&\mathbb{E}_{W^\epsilon}\left[S^\epsilon\right](W,P)+ \epsilon \sum_{d=1}^D W_d\log W_d= \nonumber\\
	    &=&\underbrace{ \sum_{d=1}^{D} W_d X_{:,d}^\dagger P^\dagger\Delta^\dagger\Delta P X_{:,d} }_{\textrm{expected Euclidean non-smoothness}} +\underbrace{ \epsilon\sum_{d=1}^D W_d\log W_d}_{\textrm{Shannon entropy regularization}},\\
	  \label{eq:functionalT2}
	 \textrm{such that } ~~ W&\in&\mathbb{P}^{(D)},P\in\mathbb{P}^{(T\times T)},\nonumber\\ 
	\mathbb{P}^{(D)} := \{ W&\in& \mathbb{R}^{D} \bigg | \; W \geq 0\; \wedge \; \sum_{d=1}^D W_{d}=1 \},\nonumber\\ 
\mathbb{P}^{(T\times T)} := \{ P&\in& \mathbb{R}^{(T\times T)} \bigg | \; P_{i,j}=\{0\}/\{1\}\; \wedge \; \sum_{i=1}^T P_{i,j}=1 \wedge \; \sum_{j=1}^T P_{i,j}=1 \}.	\nonumber\\ 
\end{eqnarray}

The following Theorem summarizes the properties of this problem's solutions: 
\begin{theorem}
\label{theorem:theone}

\begin{enumerate}[(1.)]
\item Product property: matrix product of permutation matrices is a permutation matrix;

\item Periodic translations and flipping:  Let $\left\{ W^\epsilon,P^\epsilon \right\}$ be a solution of (\ref{eq:functionalT1}-\ref{eq:functionalT2}) for given data $X$ and let $Q \in \mathbb{R}^{(T \times T)}$ be a permutation matrix. Then
\begin{enumerate}
\item the solution of (\ref{eq:functionalT1}-\ref{eq:functionalT2}) for permuted data $QX$ is given by $\left\{W^\epsilon,P^\epsilon Q^\dagger \right\}$;
\item if $Q$ is a periodic translation matrix or an order-inverting (flipping) matrix, and the finite-differencing operator $\Delta$ satisfies a periodic boundary condition, then a pair $\left\{ W^\epsilon,Q P^\epsilon \right\}$ is also a solution of (\ref{eq:functionalT1}-\ref{eq:functionalT2}) for given data $X$;
\end{enumerate}

\item Convexity of continuous relaxation: the expected non-smoothness functional \eqref{eq:functional} is convex in space of continuous double-stochastic matrices $P$;   

\item Solvability and computational cost wrt. $W$: for any fixed data $X$, permutation matrix  $P$  and regularization parameter $\epsilon>0$, constrained minimization problem (\ref{eq:functionalT1}-\ref{eq:functionalT2}) 
admits a unique closed-form solution $W^{(\epsilon)}=\left(W^{(\epsilon)}_1,\dots,W^{(\epsilon)}_D\right)$ that can be computed with the cost $\mathcal{O}(DT)$:
\begin{eqnarray}\label{eq:sol}
	W^{(\epsilon)}_d &=& \frac{\exp\left( - \epsilon^{-1} X_{:,d}^\dagger P^\dagger\Delta^\dagger\Delta P X_{:,d} \right)}{\sum_{d=1}^D\exp\left( - \epsilon^{-1} X_{:,d}^\dagger P^\dagger\Delta^\dagger\Delta P X_{:,d} \right)}\;. 
\end{eqnarray}
\item Solvability and computational cost wrt. pairwise permutations $\pi_{i,j}\in\mathbb{P}^{(T\times T)}$:  for any fixed data $X$, feature probability distribution $W$, and given permutation matrix  $P$, minimum of the pairwise permutation problem
\begin{eqnarray}\label{eq:pair}
 \left\{i^*,j^*\right\}&=&\arg\min\limits_{i,j}   \sum_{d=1}^{D} W_d X_{:,d}^\dagger (\pi_{i,j}P)^\dagger\Delta^\dagger\Delta(\pi_{i,j}P)X_{:,d}, 
\end{eqnarray}
(where $\pi_{i,j}\in\mathbb{P}^{(T\times T)}$ is a pairwise permutation matrix, that only interchanges elements $i$ and $j$, leaving all other elements unchanged) can be computed with the cost  $\mathcal{O}(DT^2)$. Moreover, $P^{new}=\pi_{i^*,j^*}P$ fulfills the condition that $F^\epsilon_{W}\left[S^\epsilon\right](W,P)\geq F^\epsilon_{W}\left[S^\epsilon\right](W,P^{new})$, where $F^\epsilon_{W}$ is defined in (\ref{eq:functionalT1}) - i.e., solving iteratively  pairwise permutation problems (\ref{eq:pair}) followed with application of (\ref{eq:sol})  leads to a monotonic minimization of the original problem (\ref{eq:functionalT1}-\ref{eq:functionalT2}) . 
\end{enumerate}
\end{theorem}

\begin{proof}
\begin{enumerate}[(1.)]
\item See, e.g., Proposition 1.5.10 (c) in \cite{artin2013algebra}.

\item Please notice that a matrix of periodic translation or flipping of the data (so-called exchange matrix) $Q$ is a special type of permutation matrix, therefore $Q^{-1} = Q^T$. Because of the statement 1 of this theorem, the matrices $P^\epsilon Q^\dagger$ and $Q P^\epsilon$ are permutation matrices and they satisfy the feasibility conditions of problem \eqref{eq:functionalT1}. Since both statements of the theorem consider unchanged $W^{\epsilon}$, we will omit constant term of Shannon entropy regularization in the following proof.

\begin{enumerate}
\item 
We assume that $P^{\epsilon}$ solves the problem (\ref{eq:functionalT1}-\ref{eq:functionalT2}) with original data $X$, i.e.,
$\forall P \in \mathbb{P}^{(T \times T)}:$
\begin{displaymath}
\sum\limits_{d=1}^D W_d X^{\dagger}_{:,d} (P^{\epsilon})^{\dagger} \Delta^{\dagger} \Delta P^{\epsilon} X_{:,d} \leq 
\sum\limits_{d=1}^D
W_d X^{\dagger}_{:,d} P^{\dagger} \Delta^{\dagger} \Delta P X_{:,d}.
\end{displaymath}
Afterwards, we utilize the property of permutation matrix $Q^\dagger Q = I$ to modify the previous inequality to the equivalent form
\begin{displaymath}
\sum\limits_{d=1}^D W_d X^{\dagger}_{:,d} Q^\dagger Q (P^{\epsilon})^{\dagger} \Delta^{\dagger} \Delta P^{\epsilon} Q^\dagger Q X_{:,d} \leq 
\sum\limits_{d=1}^D
W_d X^{\dagger}_{:,d} Q^\dagger Q P^{\dagger} \Delta^{\dagger} \Delta P Q^\dagger Q X_{:,d}.
\end{displaymath}
We denote $\hat{P} = P Q^\dagger$. Notice that $Q$ is a non-singular matrix and therefore there exist equivalency between proving the statement for all $P \in \mathbb{P}^{(T \times T)}$ and for all $\hat{P} \in \mathbb{P}^{(T \times T)}$. The previous inequality written in form $\forall \hat{P} \in \mathbb{P}^{(T \times T)}:$
\begin{displaymath}
\sum\limits_{d=1}^D W_d (Q X)^{\dagger}_{:,d} (P^{\epsilon} Q^{\dagger})^{\dagger} \Delta^{\dagger} \Delta (P^{\epsilon} Q^{\dagger}) (Q X_{:,d}) \leq
\sum\limits_{d=1}^D
W_d (Q X)^{\dagger}_{:,d} \hat{P}^{\dagger} \Delta^{\dagger} \Delta \hat{P} (Q X)_{:,d}
\end{displaymath}
proves the optimality of $P^{\epsilon} Q^{\dagger}$ for the problem (\ref{eq:functionalT1}-\ref{eq:functionalT2}) with permuted data $QX$.
 
\item Matrix $\Delta^{\dagger} \Delta$ is a Laplace matrix of the ring graph and the cyclic permutation of the vertices results in the isomorphic graph with the same Laplace matrix. Therefore, we can write 
\begin{displaymath}
Q^{\dagger} \Delta^{\dagger} \Delta Q = \Delta^{\dagger} \Delta,
\end{displaymath}
for any periodic translation matrix or exchange matrix $Q$ and the finite-differencing operator $\Delta$, which satisfies a periodic boundary condition.

Then it easy to show that objective function value of \eqref{eq:functionalT1} for $QP^{\epsilon}$ is the same as for original solution $P^{\epsilon}$
\begin{displaymath}
\begin{array}{l}
\sum\limits_{d=1}^D W_d X^{\dagger}_{:,d} (Q P^{\epsilon})^{\dagger} \Delta^{\dagger} \Delta (Q P^{\epsilon}) X_{:,d} =
\sum\limits_{d=1}^D
W_d X^{\dagger}_{:,d} (P^{\epsilon})^{\dagger} \Delta^{\dagger} \Delta P^{\epsilon} X_{d,:},
\end{array}
\end{displaymath}
which is the minimal function value from all possible $P \in \mathbb{P}^{(T \times T)}$.
\end{enumerate}
We found feasible points with the minimal function values, therefore these points are minimizers of proposed constrained problems.

\item Let $P^{\langle 0 \rangle}, P^{\langle 1 \rangle} \in \mathbb{R}^{(T \times T)}$ be two double-stochastic matrices, i.e., 
\begin{equation}
\label{eq:doublestochasticprop}
\begin{array}{lrcl}
\forall k \in \lbrace 0,1 \rbrace: ~~~ & \forall i,j: P^{\langle k \rangle}_{i,j} & \in & [0,1], \\
& P^{\langle k \rangle} \mathbbm{1} & = & \mathbbm{1}, \\
& (P^{\langle k \rangle})^T \mathbbm{1} & = & \mathbbm{1}, 
\end{array}
\end{equation}
where $\mathbbm{1} \in \mathbb{R}^T$ is vector of ones.
Then for any $\alpha \in [0,1]$ we define
\begin{displaymath}
P^{\langle \alpha \rangle} = \alpha P^{\langle 0 \rangle} + (1 - \alpha) P^{\langle 1 \rangle}.
\end{displaymath}
Since interval $[0,1]$ is a convex set, we have $\forall i,j: P^{\langle \alpha \rangle}_{i,j} \in [0,1]$.
Additionally using \eqref{eq:doublestochasticprop}, we can write
\begin{displaymath}
\begin{array}{rcl}
P^{\langle \alpha \rangle} \mathbbm{1} & = & \alpha P^{\langle 0 \rangle}\mathbbm{1} + (1 - \alpha) P^{\langle 1 \rangle}\mathbbm{1} = \alpha \mathbbm{1} + (1 - \alpha) \mathbbm{1} = \mathbbm{1}, \\
(P^{\langle \alpha \rangle})^T \mathbbm{1} & = & \alpha (P^{\langle 0 \rangle})^T \mathbbm{1} + (1 - \alpha) (P^{\langle 1 \rangle})^T \mathbbm{1} = \alpha \mathbbm{1} + (1 - \alpha) \mathbbm{1} = \mathbbm{1}, 
\end{array}
\end{displaymath}
therefore matrix $P^{\langle \alpha \rangle}$ is also double-stochastic. The set of all double-stochastic matrices is convex.

To prove the convexity of functional \eqref{eq:functionalT1} in variable $P$, it is enough to show that
\begin{equation}
\label{eq:convexitytoprove}
F^\epsilon_{W}\left[S^\epsilon\right] \left( W,\frac{1}{2}(P^{\langle 0 \rangle} + P^{\langle 1 \rangle}) \right) 
\leq 
\frac{1}{2}
\left(
F^\epsilon_{W}\left[S^\epsilon\right] ( W,P^{\langle 0 \rangle}) 
+
F^\epsilon_{W}\left[S^\epsilon\right] ( W,P^{\langle 1 \rangle}) 
\right).
\end{equation}
Since $W$ is the same for both sides of the inequality and the Shannon entropy regularization term is the same as well, we can reduce \eqref{eq:convexitytoprove} to the proof of inequalities ($\forall d = 1,\dots,D$)
\begin{equation}
\label{eq:convexitytoprove2}
\begin{array}{l}
X_{:,d}^\dagger \left( \frac{1}{2}(P^{\langle 0 \rangle} + P^{\langle 1 \rangle}) \right)^\dagger\Delta^\dagger\Delta \left( \frac{1}{2}(P^{\langle 0 \rangle} + P^{\langle 1 \rangle}) \right) X_{:,d} \leq \\[2mm]
~~~~~~~~ \leq 
\frac{1}{2}
\left(
X_{:,d}^\dagger (P^{\langle 0 \rangle})^\dagger\Delta^\dagger\Delta P^{\langle 0 \rangle} X_{:,d}
+
X_{:,d}^\dagger (P^{\langle 1 \rangle})^\dagger\Delta^\dagger\Delta P^{\langle 1 \rangle} X_{:,d}
\right).
\end{array}
\end{equation}

We start with the classification of the matrix $\Delta^\dagger\Delta \in \mathbb{R}^{T \times T}$; since for any $v \in \mathbb{R}^{T}$ we have
\begin{equation}
\label{eq:DTDsps}
  v^\dagger  \Delta^\dagger\Delta v = \Vert \Delta v \Vert_2^2 \geq 0, 
\end{equation}
the matrix $\Delta^\dagger\Delta$ is symmetric positive semidefinite. In \eqref{eq:DTDsps}, we choose $v := (P^{\langle 0 \rangle} - P^{\langle 1 \rangle}) X_{:,d}$ to obtain
\begin{displaymath}
  X_{:,d}^\dagger (P^{\langle 0 \rangle} - P^{\langle 1 \rangle})^\dagger  \Delta^\dagger\Delta (P^{\langle 0 \rangle} - P^{\langle 1 \rangle}) X_{:,d} \geq 0.
\end{displaymath}
After simple manipulations, we can rewrite this inequality to the form of
\begin{displaymath}
\begin{array}{l}
  X_{:,d}^\dagger (P^{\langle 0 \rangle})^\dagger \Delta^\dagger\Delta P^{\langle 0 \rangle} X_{:,d} 
  +
  X_{:,d}^\dagger (P^{\langle 1 \rangle})^\dagger \Delta^\dagger\Delta P^{\langle 1 \rangle} X_{:,d}
  \geq \\[2mm]
~~~~~~~~ \geq
   X_{:,d}^\dagger (P^{\langle 0 \rangle})^\dagger \Delta^\dagger\Delta P^{\langle 1 \rangle} X_{:,d} 
  +
  X_{:,d}^\dagger (P^{\langle 0 \rangle})^\dagger \Delta^\dagger\Delta P^{\langle 1 \rangle} X_{:,d}.
\end{array}
\end{displaymath}
If we add a term equal the left-hand side of inequality to both sides, we obtain
\begin{displaymath}
\begin{array}{l}
  2\left(
  X_{:,d}^\dagger (P^{\langle 0 \rangle})^\dagger \Delta^\dagger\Delta P^{\langle 0 \rangle} X_{:,d} 
  +
  X_{:,d}^\dagger (P^{\langle 1 \rangle})^\dagger \Delta^\dagger\Delta P^{\langle 1 \rangle} X_{:,d}
  \right)
  \geq \\[2mm]
~~~~~~~~ \geq
   X_{:,d}^\dagger \left( P^{\langle 0 \rangle} + P^{\langle 1 \rangle} \right)^\dagger \Delta^\dagger\Delta \left( P^{\langle 0 \rangle} + P^{\langle 1 \rangle} \right) X_{:,d}, 
\end{array}
\end{displaymath}
which is after dividing by $4$ the inequality \eqref{eq:convexitytoprove2}.

\item With fixed $X,P$ and $\epsilon >0$, the optimization problem (\ref{eq:functionalT1}-\ref{eq:functionalT2}) in variable $W$ can be written in form
\begin{displaymath}
W^{*} = \arg \min\limits_{W \in \mathbb{P}^{(D)}} W^T b + \sum\limits_{d=1}^{D} W_d \log W_d
\end{displaymath}
with constant vector $b \in \mathbb{R}^D$. In the derivation, we ignore the inequality constraints, however, the final solution will satisfy these conditions naturally. The Lagrange function corresponding to this problem is given by
\begin{displaymath}
\mathcal{L}(W,\mu) := W^T b + \sum\limits_{d=1}^{D} W_d \log W_d + \mu (W \mathbbm{1} - 1)
\end{displaymath}
with Lagrange multiplier $\mu \in \mathbb{R}$ corresponding to equality constraint.
The Karush-Kuhn-Tucker system of optimality conditions can be derived as
\begin{displaymath}
 \begin{array}{rcrcl}
 \forall d = 1,\dots,D: ~~~ \dfrac{\partial \mathcal{L}}{\partial W_d} & = & b_d + \log W_d + 1 + \mu & = & 0 \\[4mm]
 \dfrac{\partial \mathcal{L}}{\partial \mu} & = & W \mathbbm{1} - 1 & = & 0
 \end{array}
\end{displaymath}
From the first equation we get
\begin{equation}
\label{eq:WproblemKKT1}
W_d = \exp(-\mu -1) \exp(-b_d),
\end{equation}
which after the substitution into the second equation gives
\begin{displaymath}
\exp(-\mu - 1) \sum\limits_{d=1}^D \exp (-b_d) = 1
~~~ \Rightarrow
~~~
\exp(-\mu - 1) = \frac{1}{\sum\limits_{d=1}^D \exp (-b_d)}.
\end{displaymath}
The substitution of this result into \eqref{eq:WproblemKKT1} leads to \eqref{eq:sol}, which satisfies inequality constraints.

\item Let $X,W,P$ be fixed and let $\left\{i^*,j^*\right\}$ be a solution of \eqref{eq:pair}.
Then from the definition of $P^{new}=\pi_{i^*,j^*}P$ and the property of the minimizer of \eqref{eq:pair}, we have
\begin{displaymath}
F^\epsilon_{W}\left[S^\epsilon\right](W,P^{new}) =
F^\epsilon_{W}\left[S^\epsilon\right](W,\pi_{i^*,j^*}P)
\leq
F^\epsilon_{W}\left[S^\epsilon\right](W,\pi_{i,j}P)
\end{displaymath}
for any $i,j \in \lbrace 1,\dots,T \rbrace$. We choose $i=j$ to obtain $\pi_{i,j} = I$ (identity matrix) and consequently the right-hand side of previous inequality can be written as
\begin{displaymath}
F^\epsilon_{W}\left[S^\epsilon\right](W,\pi_{i,j}P) = F^\epsilon_{W}\left[S^\epsilon\right](W,P).
\end{displaymath}
This proves that the iterative procedure based on the construction of locally-optimal sequence of pairwise permutation matrices cannot increase of the function value of the objective function \eqref{eq:functionalT1}.

Let us discuss the computational complexity of solving \eqref{eq:pair}. 
The algorithm computes the non-smoothness between all points in all individual dimensions. Such a computation is $\mathcal{O}(DT^2)$. During this computation, the algorithm identifies the optimal permutation, i.e., the solution of the problem \eqref{eq:pair}. Values of non-smoothness with new permutations are reused in the following $W$-step \eqref{eq:sol}, however, after a new value of $W$ is identified, the local non-smoothness values have to be recomputed.

\end{enumerate}
\end{proof}

\section{Numerical approach}

The monotonic iterative minimization of the original problem (\ref{eq:functionalT1}-\ref{eq:functionalT2}) can be achieved with the following Algorithm 1, having numerical iteration cost scaling  $\mathcal{O}(DT^2)$ in the leading order (see the above proof). For a given data matrix $X$, the algorithm can be performed for various values of $\epsilon$ and with different initializations for $W$ and $P$. The optimal value of $\epsilon$ and the best possible solution to  (\ref{eq:functionalT1}-\ref{eq:functionalT2}) can be identified using, for example, the popular L-curve method that detects the optimal solution of regularization problems \cite{calvetti00} - in the considered case as an elbow-point of the curve featuring pairs of values for the expected non-smoothness $\mathbb{E}_{W^\epsilon}\left[S^\epsilon\right]$ vs. the Shannon entropy $\sum_{d=1}^D W^\epsilon_d\log W^\epsilon_d$ for different values of $\epsilon$  (see Fig 1D for an example). 

 \begin{algorithm}
 \caption{for the solution of optimization problem (\ref{eq:functionalT1}-\ref{eq:functionalT2})}\label{alg:eSPAgeneric}
 \begin{algorithmic}[1]
   \Require Data matrix $X \in \mathbb{R}^{T,D}$, hyperparameter $\epsilon_W > 0$
   \Ensure Optimal values of $W$ and $P$ minimizing the functional $F^\epsilon_{W}$ in \eqref{eq:functionalT1} and satisfying constraints \eqref{eq:functionalT2}
   \State randomly choose initial $W^{(1)}$ and $P^{(1)}$\;
   \State $I=1;\, \left\{F^\epsilon_{W}\right\}^{(I)}=\infty;\, \delta\left\{F^\epsilon_{W}\right\}^{(I)}=\infty$\;
 \While{$\delta\left\{F^\epsilon_{W}\right\}^{(I)}>tol$}
 \State {\underline{$P$-step}: find $\pi_{i^*,j^*}$ as a solution of  \eqref{eq:pair} for fixed  $W^{(I)}$ and $P^{(I)}$
  \State $P^{(I+1)}=\pi_{i^*,j^*}P^{(I)}$\;
   \vspace{0.2cm} \;
  \State \underline{$W$-step}: find $W^{(I+1)}$ evaluating the explicit solution (\ref{eq:sol})  for a fixed  $P^{(I)}$\;
  \vspace{0.2cm} \;
  \State $\left\{F^\epsilon_{W}\right\}^{(I+1)}=F^\epsilon_{W}\left[S^\epsilon\right](W^{(I+1)},P^{(I+1)})$\;
  \State $I = I + 1$\;
  \State $\delta\left\{F^\epsilon_{W}\right\}^{(I)}= \left\{F^\epsilon_{W}\right\}^{(I-1)}-\left\{F^\epsilon_{W}\right\}^{(I)}$\; 
 }
 \EndWhile

\end{algorithmic}
\end{algorithm}

Extension of the common monotonic sorting to multiple dimensions can also be achieved using the formulation (\ref{eq:functionalT1}), by rewriting it  in the following simplified form that is linear in $P$ and log-linear in $W$: 
\begin{eqnarray}\label{eq:functionalM1}
	 \left\{W^\epsilon,P^\epsilon\right\}&=&\arg\min\limits_{W\in\mathbb{P}^{D},P\in\mathbb{P}^{T\times T}} M^\epsilon_{W}\left[S^\epsilon\right] (W,P), \nonumber\\
	  \textrm{where} ~M^\epsilon_{W}\left[S^\epsilon\right](W,P)&=&\underbrace{- \bar{T}^\dagger P\sum_{d=1}^{D} W_d X_{:,d} }_{\textrm{expected monotonicity}} +\underbrace{ \epsilon\sum_{d=1}^D W_d\log W_d}_{\textrm{Shannon entropy regularization}},
\end{eqnarray}
where $\bar{T}\in\mathbb{R}^{(T\times 1)} $ is a monotonic (increasing or decreasing) sequence of numbers, e.g.,  $\bar{T}=(1,2,\dots,T)$. It is straightforward to validate that for a fixed distribution of dimension weights $W_1,\dots,W_D$, $A\equiv\bar{T}$,  $B\equiv\sum_{d=1}^{D} W_d X_{:,d}$ and $\epsilon_{W}\equiv0$, above problem (\ref{eq:functionalM1}) is equivalent to a problem of matching of a given vectorized graph Laplacian  $B$ to a vectorized reference graph Laplacian $A$ in graph-matching and -allignment problems, see, e.g., eq (5)  in \cite{mena17}:
\begin{equation}
\label{eq:functionalM1}
	 \arg\min\limits_{P\in\mathbb{P}^{T\times T}} \|A-PB\|_2^2=\arg\min\limits_{P\in\mathbb{P}^{T\times T}} \left(-\text{trace} \left[A^\dagger PB\right]\right),
\end{equation}
where we use the fact that the 2-norm is invariant with respect to the 'crisp' permutations of the vector components. 
Furthermore, it should be noted that the aforementioned equality property \eqref{eq:functionalM1} is not preserved when employing the concept of 'soft' permutations $P$.

Then, if $\bar{T}$ is a monotonic sequence of integers (ascending or descending), deploying 
Theorem \ref{theorem:theone} and the Master theorem  \cite{master80,cormen01}, it is very straightforward to demonstrate that problem (\ref{eq:functionalM1}) subject to constraints (\ref{eq:functionalT2}) for a fixed $W$ is equivalent to a standard one-dimensional sorting problem, and can be efficiently solved using  the divide-and-conquer strategy with the modification of Algorithm 1. Thereby,  the $P$-step (monotonic sorting of $B$) can be performed with QuickSort-algorithm, resulting in the overall iteration complexity scaling of $\mathcal{O}(T\log(T)+DT)$  \cite{cormen01}.    

\section{Results}

\subsection{Synthetic dataset}

To illustrate and to test the suggested method for solving  (\ref{eq:functionalT1} - \ref{eq:functionalT2}), we first apply it to the analysis of synthetic data featuring two first dimensions with two periodic sinus and cosinus signals of different period and phase, and with all other feature dimensions (3 to 52) being the random realizations of the uniformly distributed random variable of the same variance as the first two dimensions  (Fig. 1A). Next, we randomly permute the rows of this data matrix along the data index dimension $t=1,\dots,T$, resulting in a dataset in Fig. 1B. Recovering the original smooth ordering from this scrambled data matrix $X$ in Fig.1B would require solving two problems simultaneously: (i) identifying that only the first two data dimensions out of 52 contain the permutation of two smooth features (requiring to check ${52 \choose 2}$ dimension combinations); and (ii) finding correct data permutations  - out of possible $50!$ permutations - in these two particular dimensions. Applying full combinatorial search to detect the smooth original patterns from Fig. 1A would require checking ${52 \choose 2}\cdot 50!=4\cdot 10^{67}$ possible data rearrangements - which is factor two larger than the number of atoms in our galaxy. Using the currently most powerful top-1 supercomputer (``Frontier'' at Oak Ridge National Lab, $8.7$ Mio cores and 21'100kW of energy consumption) to check all these possibilities - with $10^{-4}$ sec. for checking one possibility on one core - would require $1.93$ times more energy than was produced by all of the stars in the visible part of the universe since the Big Bang. But also deploying common ML/AI algorithms for feature selection (e.g., based on linear or non-linear covariance and correlation) would not help: as can be seen from Fig. 1B, the correlation matrix is diagonal and no statistically-significant cross-correlations can be detected in these data. In contrast, the numerical solution of the optimization problem  (\ref{eq:functionalT1}-\ref{eq:functionalT2}) with Algorithm 1 proposed above allows finding the correct dimensions and the permutation recovering the original smooth patterns (see the red panel in Fig.1D) in several seconds on the commodity laptop. Repeating Algorithm 1 for different values of regularization parameter $\epsilon$, one finds the best solution as an elbow-point of the respective L-curve, featuring pairs of values for expected non-smoothness $\mathbb{E}_{W^\epsilon}\left[S^\epsilon\right]$ vs. the Shannon entropy $\sum_{d=1}^D W^\epsilon_d\log W^\epsilon_d$ (see Fig. 1D).   As can be seen from Figs. 1E and 1F, the overall computational cost of the proposed algorithm scales linearly in data dimension and allows robust detection of smooth pattern permutations, also in situations when the data statistics size $T$ is smaller than the feature data dimension $D$.      

\subsection{Taiwanese companies bankruptcy data analysis}

In order to demonstrate the versatility of the proposed methodology in identifying pattern permutations possessing specific characteristics, we examine an economic data scenario focused on identifying a trajectory that mitigates the risk of bankruptcy for a company. This trajectory aims to exhibit a smooth decrease in bankruptcy risk while requiring minimal alterations to the company's characteristics at each step. Notably, this analysis will exclusively utilize data instances representing other companies.
The minimal amount of changes in every step can be quantified and minimized with the expected non-smoothness measure $\mathbb{E}_{W^\epsilon}\left[S^\epsilon\right]$ proposed in (\ref{eq:functionalT1}), whereas desired smoothness $W_0$ in the risk dimension can be imposed with an additional linear constraint  
\begin{eqnarray}\label{eq:risk_monotone}
W_r&=&W_0,
\end{eqnarray}
  where $r$ is the index of risk dimension in the data. 
  
Our study commences by considering a dataset encompassing information from 6819 Taiwanese companies. Besides of the $95$ company characteristics, it contains a label dimension with values $0$ or $1$, dependent on the fact whether the particular company went bankrupt within one year or not  \cite{taiwan16}. 
  To infer an additional bankruptcy risk dimension for the following analysis, we deploy various AI/ML methods (including gradient-boosted random forests, deep and shallow neuronal networks with various architectures, and entropic learning methods), comparing their Area Under Curve (AUC) bankruptcy prediction quality on the test data that was not used in training \cite{xgboost16,lecun15,Horenko_2020,espa_22,horenko_pnas_22,horenko_pnas_23}. Afterwards, we use the bankruptcy probabilities from the eSPA+ entropic learning method as an additional feature dimension - since eSPA+ was showing the highest AUC ($0.975\pm0.013$) with a most narrow confidence interval on test data (see Fig. 1G and 1H).  Solving  (\ref{eq:functionalT1}-\ref{eq:functionalT2}) subject to additional constraints (\ref{eq:risk_monotone})  - and with $W_0=0.1$ for one of the companies that actually went bankrupt (and with eSPA+ bankruptcy risk of 0.58) - results in the smooth and risk-monotonic trajectory reducing the risk to zero in $7$ steps, where these $7$ steps are represented by $7$ other companies that are available in the same data (see Fig. 1I). Increasing $W_0$ results in the increasing number of risk-reducing steps - but also in the decreasing smoothness of permutations in the original $95$ company dimensions. Decreasing $W_0$ results in the quick growth of non-smoothness in the risk dimension - so empirically $W_0=0.1$ represents an optimal elbow point that achieves the best smoothness both in risk and in the other company features.             


\begin{figure}[h!]
        \includegraphics[clip,  width=1.1\textwidth]{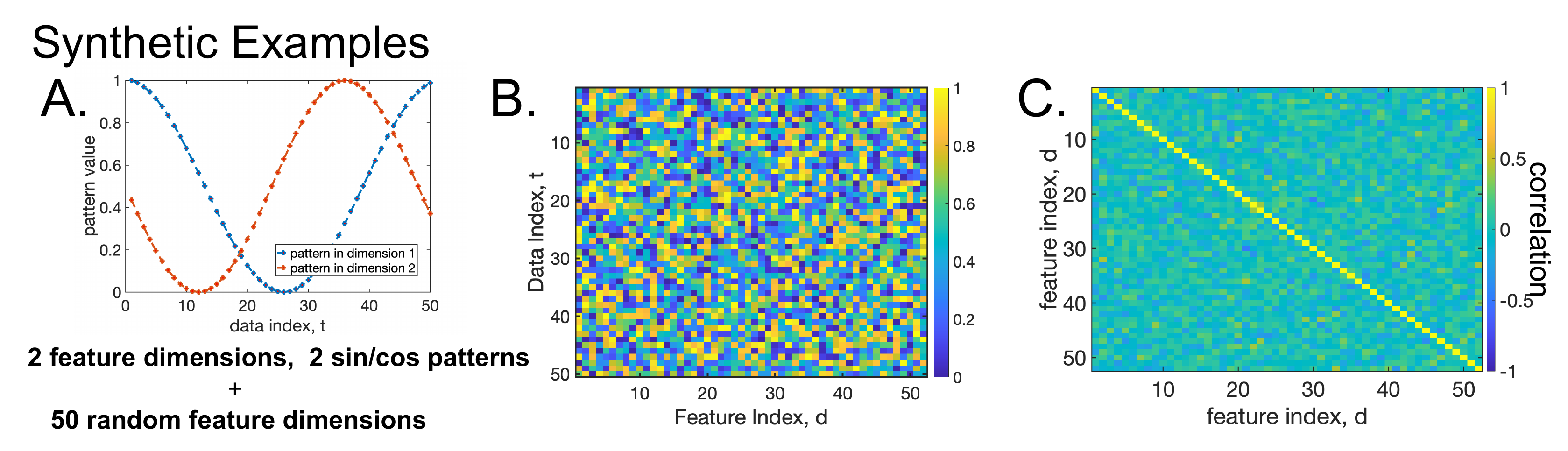}
        \includegraphics[clip,  width=1.1\textwidth]{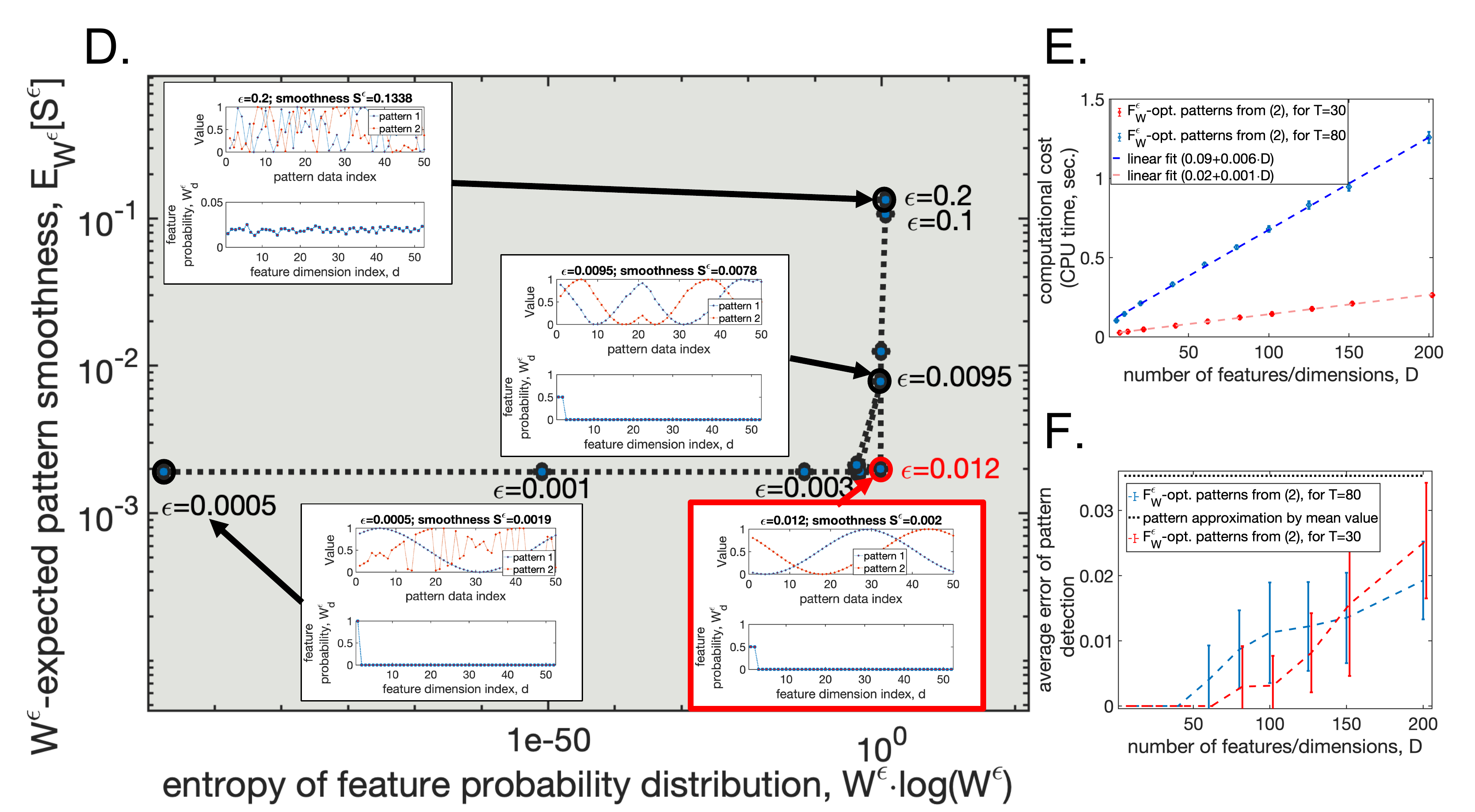}
        \includegraphics[clip,  width=1.1\textwidth]{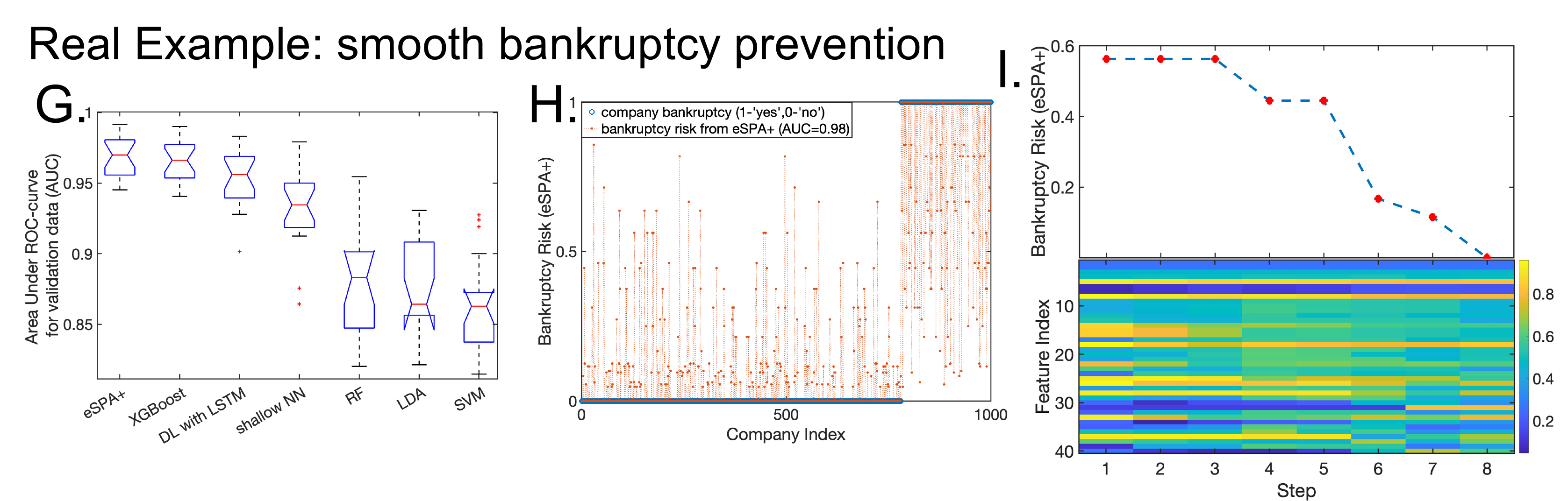}
 \caption{Recovering smooth data ordering for the synthetic data examples (A-F) and for the multidimensional data set on bankruptcy of Taiwanese companies (G-I) \cite{taiwan16}.}
\end{figure}





\bmhead{Availability of data and material} The economical data used in Fig. 1G-1I is available at \url{https://www.kaggle.com/datasets/fedesoriano/company-bankruptcy-prediction}. 

\bmhead{Code availability} No.

\section*{Declarations}

\bmhead{Confict of interest} The authors have no confict of interest to declare that are relevant to the content of this article.

\bmhead{Ethics approval} Not applicable.

\bmhead{Consent to participate} Not applicable.

\bmhead{Consent for publication} Not applicable

\bibliography{SER}

\end{document}